\documentclass[11pt]{article}
\usepackage{graphicx}
\usepackage{fullpage}
\usepackage{xspace}
\usepackage{xcolor}
\usepackage{booktabs}
\usepackage{multirow}
\usepackage{multicol}
\usepackage{amssymb,amsfonts,amsmath,amsthm}
\usepackage{colonequals}
\usepackage{algorithm}
\usepackage{algorithmicx}
\usepackage{algpseudocode}
\usepackage{secdot}
\usepackage{enumitem}
\usepackage{authblk}
\usepackage{changepage} 
\usepackage{soul}
\usepackage{newfloat}
\usepackage{hyperref}
\DeclareFloatingEnvironment[name={Supplementary Figure},fileext=lof]{suppfigure}

\usepackage[backend=bibtex,minnames=3,maxnames=4,sorting=none,doi=false,firstinits=true,isbn=false,url=false,citestyle=numeric-comp]{biblatex}

\def\R{{\mathbb{R}}}

\def\E{{\mathbb E}}
\def\X{{\mathcal X}}

\newtheorem{thm}{Theorem}

\newtheorem{defn}[thm]{Definition}

\DeclareMathOperator*{\argmax}{arg\,max}

\newcommand{\pluseq}{\mathrel{+}=}
\newcommand{\minuseq}{\mathrel{-}=}

\makeatletter
\renewcommand*{\ALG@name}{Perceptron}
\makeatother

\renewbibmacro{in:}{} 
\bibliography{continual_learning_manuscript.bib} 
\DeclareNameAlias{sortname}{last-first} 
\DeclareNameAlias{default}{last-first} 

\newcommand\flymodel{FlyModel\xspace}
\newcommand\mnist{MNIST-20\xspace}

\title{Algorithmic insights on continual learning from fruit flies}

\author[1]{Yang Shen}
\author[2]{Sanjoy Dasgupta}
\author[1]{Saket Navlakha}
\affil[1]{Cold Spring Harbor Laboratory, Simons Center for Quantitative Biology, Cold Spring Harbor, NY}
\affil[2]{Computer Science and Engineering Department, University of California San Diego, La Jolla, CA}
\date{}

\begin{document}
\maketitle

\begin{abstract}
\noindent Continual learning in computational systems is challenging due to catastrophic forgetting. We discovered a two-layer neural circuit in the fruit fly olfactory system that addresses this challenge by uniquely combining sparse coding and associative learning. In the first layer, odors are encoded using sparse, high-dimensional representations, which reduces memory interference by activating non-overlapping populations of neurons for different odors. In the second layer, only the synapses between odor-activated neurons and the odor’s associated output neuron are modified during learning; the rest of the weights are frozen to prevent unrelated memories from being overwritten. We show empirically and analytically that this simple and light-weight algorithm significantly boosts continual learning performance. The fly’s associative learning algorithm is strikingly similar to the classic perceptron learning algorithm, albeit two modifications, which we show are critical for reducing catastrophic forgetting. Overall, fruit flies evolved an efficient life-long learning algorithm, and circuit mechanisms from neuroscience can be translated to improve machine computation.

\end{abstract}

\clearpage
\section*{Introduction}


Catastrophic forgetting --- i.e., when neural networks inadvertently overwrite old memories with new memories --- remains a long-standing problem in machine learning~\cite{Parisi2019}. Here, we studied how fruit flies learn continuously to associate odors with behaviors and discovered a circuit motif capable of alleviating catastrophic forgetting.


While much attention has been paid towards learning good representations for inputs, an equally challenging problem in continual learning is finding good ways to preserve associations between these representations and output classes. Indeed, modern deep networks excel at learning complex and discriminating representations for many data types, which in some cases have resulted in super-human classification performance~\cite{Lecun2015}. However, these same networks are considerably degraded when classes are learned sequentially (one at a time), as opposed to being randomly interleaved in the training data~\cite{Ratcliff1990,McClelland1995,French1999}. The effect of this simple change is profound and has warranted the search for new mechanisms that can preserve input-output associations over long periods of time.


Since learning in the natural world often occurs sequentially, the past few years have witnessed an explosion of brain-inspired continual learning models. These models can be divided into three categories: 1) regularization models, where important weights (synaptic strengths) are identified and protected~\cite{Hinton1987,Fusi2005,Benna2016,Kirkpatrick2017,Zenke2017}; 2) experience replay models, which use external memory to store and re-activate old data~\cite{lopez2017gradient}, or which use a generative model to generate new data from prior experience~\cite{Ven2020,Tadros2020,shin2017continual}; and 3) complementary learning systems~\cite{McClelland1995,Roxin2013}, which partition memory storage into multiple sub-networks, each subject to different learning rules and rates. Importantly, these models often take inspiration from mammalian memory systems, such as the hippocampus~\cite{Wilson1994,Rasch2007} or the neocortex~\cite{Qin1997,Ji2007}, where detailed circuit anatomy and physiology are still lacking. Fortunately, continual learning is also faced by simpler organisms, such as insects, where supporting circuit mechanisms are understood at synaptic resolution~\cite{Takemura2017,Zheng2018,Li2020}.


Here, we developed an algorithm to reduce catastrophic forgetting by taking inspiration from the fruit fly olfactory system. This algorithm stitches together three well-known computational ideas --- sparse coding~\cite{Maurer2013,Ruvolo2013,Ororbia2019,Subutai2019,Rapp2020,Hitron2020}, synaptic freezing~\cite{Hinton1987,Fusi2005,Benna2016,Kirkpatrick2017,Zenke2017}, and perceptron-style learning~\cite{Minsky1988} --- in a unique and effective way, which we show boosts continual learning performance compared to alternative algorithms. Importantly, the \flymodel uses neurally-consistent associative learning and does not require backpropagation. Finally, we show that the fruit fly circuit performs better than alternative circuits in design space (e.g., replacing sparse coding with dense coding, associative learning with supervised learning, freezing synapses with not freezing synapses), which provides biological insight into the function of these evolved circuit motifs and how they operate together in the brain to sustain memories.

\section*{Results}

\subsection*{Circuit mechanisms for continual learning in fruit flies} 

How do fruit flies associate odors (inputs) with behaviors (classes) such that behaviors for odors learned long ago are not erased by newly learned odors? We first review the basic anatomy and physiology of two layers of the olfactory system that are relevant to the exposition here. For a more complete description of this circuit, see Modi et al.~\cite{Modi2020}.

The two-layer neural circuit we study takes as input an odor after a series of pre-processing steps have been applied. These steps begin at the sensory layer and include gain control~\cite{Root2008,Gorur2017}, noise reduction~\cite{Wilson2013}, and normalization~\cite{olsen2010divisive,Stevens2015}. After these steps, odors are represented by the firing rates of $d=50$ types of projection neurons (PNs), which constitute the input to the two-layer network motif described next.\\

\noindent \textbf{Sparse coding.} The goal of the first layer is to convert the dense input representation of the PNs into a sparse, high-dimensional representation~\cite{Cayco2019} (Figure~\ref{fig:overview}A). This is accomplished by a set of about 2000 Kenyon cells (KCs), which receive input from the PNs. The matrix connecting PNs to KCs is sparse and approximately random~\cite{caron2013random}; i.e., each KC randomly samples from about 6 of the 50 projection neurons and sums up their firing rates. Next, each KC provides feed-forward excitation to a single inhibitory neuron, called APL. In return, APL sends feed-back inhibition to each KC. The result of this loop is that approximately 95\% of the lowest-firing KCs are shut off, and the top 5\% remain firing, in what is often referred to as a winner-take-all (WTA) computation~\cite{Turner2008,Lin2014,Stevens2015}. Thus, an odor initially represented as a point in $\mathbb{R}^{50}_+$ is transformed, via a 40-fold dimensionality expansion followed by WTA thresholding, to a point in $\mathbb{R}_+^{2000}$, where only approximately 100 of the 2000 KCs are active (i.e., non-zero) for any given odor. 

This transformation was previously studied in the context of similarity search~\cite{Dasgupta2017,dasgupta2018neural,Papadimitriou2018,Ryali2020}, compressed sensing~\cite{Stevens2015,Zhang2016}, and pattern separation for subsequent learning~\cite{babadi2014sparseness,LitwinKumar2017,dasgupta2020expressivity}.\\

\noindent \textbf{Associative learning.} The goal of the second layer is to associate odors (sparse points in high-dimensional space) with behaviors.  In the fly, this is accomplished by a set of 34 mushroom body output neurons (MBONs~\cite{aso2014neuronal}), which receive input from the 2000 KCs, and then ultimately connect downstream onto motor neurons that drive behavior. Our focus will be on a subset of MBONs that encode the learned valence (class) of an odor. For example, there is an ``approach'' MBON that strongly responds if the odor was previously associated with a reward, and there is an ``avoid'' MBON that responds if the odor was associated with punishment~\cite{Hige2015}. Thus, the sparse, high-dimensional odor representation of the KCs is read-out by a smaller set of MBONs that encode behaviorally-relevant odor information important for decision-making.

The main locus of associative learning lies at the synapses between KCs and MBONs (Figure~\ref{fig:overview}). During training, say the fly is presented with a naive odor (odor A) that is paired with a punishment (e.g., an electric shock). How does the fly learn to avoid odor A in the future? Initially, the synapses from KCs activated by odor A to both the ``approach'' MBON and the ``avoid'' MBON have equal weights. When odor A is paired with punishment, the KCs representing odor A are activated around the same time that a punishment-signaling dopamine neuron fires in response to the shock. The released dopamine causes the synaptic strength between odor A KCs and the approach MBON to decrease, resulting in a net increase in the avoidance MBON response\footnote{Curiously, approach behaviors are learned by decreasing the avoid MBON response, as opposed to increasing the approach MBON response, as may be more intuitive.}. Eventually, the synaptic weights between odor A KCs and the approach MBON are sufficiently reduced to reliably learn the avoidance association \cite{Felsenberg2018}.

Importantly, the only synapses that are modified in each associative learning trial are those from odor A KCs to the approach MBON. All synapses from odor A KCs to the avoid MBON are frozen (i.e., left unchanged), as are all weights from silent KCs to both MBONs. Thus, the vast majority of synapses are frozen during any single odor-association trial. How is this implemented biologically? MBONs lie in physically separated ``compartments''~\cite{aso2014neuronal}. Each compartment has its own dopamine neurons, which only modulate KC$\rightarrow$MBON synapses that lie in the same compartment. In the example above, a punishment-signaling dopamine neuron lies in the same compartment as the approach MBON and only modifies synapses between active KCs and the approach MBON. Similarly, a reward-signaling dopamine neuron lies in the same compartment as the avoid MBON~\cite{Felsenberg2018}. Dopamine released in one compartment does not ``spillover'' to affect KC$\rightarrow$MBON synapses in neighboring compartments, allowing for compartment-specific learning rules~\cite{Aso2016}. 

To summarize, associative learning in the fly is driven by dopamine signals that only affect the synapses of sparse odor-activated KCs and a target MBON that drives behavior.

\subsection*{The \flymodel}

We now introduce a continual learning algorithm based on the two-layer olfactory circuit described above. 

As input, we are given a $d$-dimensional vector, $x = (x_1, x_2, \dots, x_d) \in\mathbb{R}^d$ (analogous to the projection neuron firing rates for an odor). As in the fly circuit, we assume that $x$ is pre-processed to remove noise and encode discriminative features. For example, when inputs are images, we could first pass each image through a deep network and use the representation in the penultimate layer of the network as input to our two-layer circuit. Pre-processing is essential when the original data are noisy and not well-separated but could be omitted in simpler datasets with more prominent separation among classes. To emphasize, our goal here is not to study the complexities of learning good representations, but rather to develop robust ways to associate inputs with outputs.

The first layer computes a sparse, high-dimensional representation of $x$. This layer consists of $m$ units (analogous to Kenyon cells), where $m \approx 40d$ (analogous to the expansion from 50 PNs to 2000 KCs). The input layer and the first layer are connected by a sparse, binary random matrix, $\Theta$, of size $m \times d$. Each column of $\Theta$ contains about $0.1d$ ones in random positions (analogous to each KC sampling from 6 of the 50 PNs), and the rest of the positions in the column are set to zero. The initial KC representation $\psi(x) = (\psi_1, \psi_2, \dots, \psi_m) \in \mathbb{R}^m$ is computed as:
\begin{align} 
    \psi(x) = \Theta x. \label{eqn:1}
\end{align}
After this dimensionality expansion, a winner-take-all process is applied, so that only the top $l$ most active KCs remain on, and the rest of the KCs are set to 0 (in the fly, $l=100$, since only 5\% of the 2000 KCs are left active after APL inhibition). This produces a sparse KC representation $\phi(x) = (\phi_1, \phi_2, \dots, \phi_m) \in \mathbb{R}^m$, where:
\begin{align} \label{eqn:2}
    \phi_{i} = 
    \begin{cases}
    \psi_{i} & \text{if $\psi_{i}$ is one of the $l$ largest positive entries of $\psi(x)$ } \\
    0 & \text{otherwise.}
    \end{cases}
\end{align}
For computational convenience, a min-max normalization is applied to $\phi(x)$ so that each KC has a value between 0 and 1. The matrix $\Theta$ is fixed and not modified during learning; i.e., there are no trainable parameters in the first layer.

The second layer is an associative learning layer, which contains $k$ output class units, $y = \{y_1, y_2, \dots, y_k\}$ (analogous to MBONs). The $m$ KCs and the $k$ MBONs are connected with all-to-all synapses. Say an input $x$ is to be associated with target MBON $j$. When $x$ arrives, a hypothetical dopamine neuron signaling class $j$ is activated at the same time, so that the only synapses that are modified are those between the KCs active in $\phi(x)$ and the $j$\textsuperscript{th} MBON. No other synapses --- including those from the active KCs in $\phi(x)$ to the other $k-1$ MBONs --- are modified. We refer to this as ``partial freezing'' of synaptic weights during learning. 

Formally, let $w_{ij} \in [0, 1]$ be the synaptic weight from KC $i$ to MBON $j$. Then, for all $i \in [1\dots m], j \in [1\dots k]$, the weight update rule after each input $x$ is:
\begin{align} \label{eqn:weights}
    w_{ij} = 
    \begin{cases}
    (1-\alpha)w_{ij} + \beta \phi_i & \text{if $j = \textrm{target}$} \\
    (1-\alpha)w_{ij} & \text{otherwise.}
    \end{cases}
\end{align}
Here, $\beta$ is the learning rate, and $\alpha$ is a very small forgetting term that mimics slow, background memory decay. In our experiments, we set $\alpha=0$ to minimize forgetting and to simplify the model. The problem of weight saturation arises when $\alpha=0$, since weights can only increase, and never decrease. However, despite tens of thousands of training steps, the vast majority of weights did not saturate since most KCs are inactive for most classes (sparse coding) and only a small fraction of the active KC synapses are modified during learning (partial freezing). Nonetheless, in practice, some small, non-zero $\alpha$ may be desired to avoid every synapse from eventually saturating.

Finally, biological synaptic weights have physical bounds on their strength, and here we mimic these bounds by capping weights to $[0,1]$.\\

\noindent \textbf{Similarities and differences to the fruit fly olfactory circuit.} The \flymodel is based on two core features of the fruit fly olfactory circuit: sparse coding and partial freezing of synaptic weights during learning. There are, however, additional complexities of this ``two-layer'' olfactory circuit that we do not consider here. First, there are additional recurrent connections in the circuit, including KC$\leftrightarrow$KC connections~\cite{Eichler2017} and dopamine$\leftrightarrow$MBON connections~\cite{Cervantes2017,Takemura2017}; further, there is an extensive, four-layer network of interactions amongst MBONs~\cite{aso2014neuronal}, the function of which still remains largely unknown. Second, we assume KCs and MBONs make all-to-all connections, whereas in reality, each MBON is typically connected to less than half of the KCs. KCs are divided into distinct lobes that innervate different compartments~\cite{aso2014neuronal}. Lobes and compartments allow for an odor's KC representation to be ``split'' into a parallel memory architecture, where each compartment has a different memory storage capacity, update flexibility, and retention and decay rates~\cite{Aso2016}. Third, we assumed that co-activation of dopamine neurons and KCs increases the synaptic weights between KCs and the target MBON, whereas in reality, when learning to avoid an odor, the strength of response to the opposite behavior (approach MBON) is decreased. Conceptually, the net effect is equivalent for binary classification, but the latter leads to additional weight interference when there are $> 2$ classes because it requires decreasing weights to all non-target MBONs. 

We excluded these additional features to reduce the number of model parameters and to sharpen our focus on the two core features mentioned above. These abstractions are in line with those made by previous models of this circuit (e.g.~\cite{Stevens2015,Peng2017,Mittal2020}). Some of these additional features may be useful in more sophisticated continual learning problems that are beyond the scope of our work here (Discussion).

\subsection*{Testing framework and problem setup} 

We tested each algorithm on two datasets using a class-incremental learning setup~\cite{farquhar2019towards,Ven2020}, in which the training data was ordered and split into sequential tasks. For the \mnist dataset (a combination of regular MNIST and Fashion MNIST; Methods), we used 10 non-overlapping tasks, where each task is a classification problem between two classes. For example, the first task is to classify between digits 0 and 1, the second task is to classify digits 2 and 3, etc. Similarly, the CIFAR-100 dataset (Methods) is divided into 25 non-overlapping tasks, where each task is a classification problem among four classes. In each task, all instances of one class are presented sequentially, followed by all instances of the second class. Only a single pass is made through the training data (i.e., one epoch) to mimic an online learning problem.

Testing is performed after the completion of training of each task, and is quantified using two measures. The first measure --- the accuracy for classes trained so far --- assesses how well classes from previous tasks remain correctly classified after a new task is learned. Specifically, after training task $i$, we report the accuracy of the model when tested on classes from all tasks $\leq i$. For example, say a model has been trained on the first three tasks --- classify 0 vs.\@ 1, 2 vs.\@ 3, and 4 vs.\@ 5. During the test phase of task three, the model is presented with test examples from digits 0--5, and their accuracy is reported. The second measure --- memory loss --- quantifies forgetting for each task separately. We define the memory loss of task $i$ as the accuracy of the model when tested (on classes from task $i$ only) immediately after training on task $i$ minus the accuracy when tested (again, on classes from task $i$ only) after training on all tasks, i.e., at the end of the experiment. For example, say the immediate accuracy of task $i$ is 0.80, and the accuracy of task $i$ at the end of the experiment is 0.70. Then the memory loss of task $i$ is 0.10. A memory loss of zero means that the memory of the task was perfectly preserved despite learning new tasks.\\

\noindent \textbf{Comparison to other methods.} We compared the \flymodel with five methods, briefly described below:
\begin{enumerate}
    \item \emph{Elastic weight consolidation} (EWC~\cite{Kirkpatrick2017}) uses the Fisher information criterion to identify weights that are important for previously learned tasks, and then introduces a penalty if these weights are modified when learning a new task.
    \item \emph{Gradient episodic memory} (GEM~\cite{lopez2017gradient}) uses a memory system that stores a subset of data from previously learned tasks. These data are used to assess how much the loss function on previous tasks increases when model parameters are updated for a new task.
    \item \emph{Brain-inspired replay} (BI-R~\cite{Ven2020}) protects old memories by using a generative model to replay activity patterns related to previously learned tasks. The replayed patterns are generated using feedback connections, without storing data.
    \item \emph{Vanilla} is a standard fully-connected neural network that does not have any explicit continual learning mechanism. This is used as a lower bound on performance.
    \item \emph{Offline} is a standard fully-connected neural network, but instead of learning tasks sequentially, it is presented with all classes from the tasks in a random order. For example, for the third task on \mnist, Offline is trained with digits 0--5 randomly shuffled. Then, for the fourth task, Offline is re-trained from scratch on digits 0--7. The Offline model is used as an upper bound on performance.
\end{enumerate}
\noindent All five of these methods use backpropagation for training weights (both PN$\rightarrow$KC weights and KC$\rightarrow$MBON weights). In addition, all five methods (except BI-R; Methods) use the same architecture as the \flymodel  --- the same number of layers, the same number of units per layer ($m$ KCs in the first layer, $k$ MBONs in the second layer) --- and they all use the same hidden unit activation function (ReLU). Finally, for a fair comparison, all methods, including the \flymodel, use the same representation for each input. Thus, the primary difference amongst methods is how learning mechanisms store and preserve memories.

\subsection*{The \flymodel outperforms existing methods in class-incremental learning}

The \flymodel reduced catastrophic forgetting compared to all four continual learning methods tested. For example, on the \mnist dataset (Figure~\ref{fig:continual}A), after training on 5 tasks (10 classes), the accuracy of the \flymodel was 0.86 $\pm$ 0.0006 compared to 0.77 $\pm$ 0.02 for BI-R, 0.69 $\pm$ 0.02 for GEM, 0.58 $\pm$ 0.10 for EWC, and 0.19 $\pm$ 0.0003 for Vanilla. At the end of training (10 tasks, 20 classes trained), the test accuracy of the \flymodel was at least 0.19 higher than any other method, and only 0.11 lower than the optimal Offline model, which is trained using all classes presented together, instead of sequentially.

Next, we used the memory loss measure (Methods) to quantify how well the ``memory'' of an old task is preserved after training new tasks (Figure~\ref{fig:continual}B, Figure~S1). As expected, the standard neural network (Vanilla) preserves almost no memory of previous tasks; i.e., it has a memory loss of nearly 1 for all tasks except the most recent task. While GEM, EWC, and BI-R perform better --- memory losses of 0.24, 0.27, and 0.42, respectively, averaged across all tasks --- the \flymodel has an average memory loss of only 0.07. This means that the accuracy of task $i$ was only degraded on average by 7\% at the end of training when using the \flymodel. 

Similar trends were observed on a second, more difficult dataset (CIFAR-100; Figure~\ref{fig:continual}C--D), where the \flymodel had an accuracy that was at least 0.15 greater than all continual learning methods, and performed only 0.13 worse than the Offline model.

\subsection*{Sparse coding and partial freezing are both required for continual learning}

An important challenge in theoretical neuroscience is to understand why circuits may be designed the way they are. Quantifying how evolved circuits fare against putative, alternative circuits in design space could provide insight into the biological function of observed network motifs. We first explored this question in the context of the two core components in the \flymodel: sparse coding of representations in the first layer, and partial freezing of synaptic weights in the associative learning layer. Are both of these components required, or can good performance be attained with only one or the other? 

We piecemeal explored the effects of replacing sparse coding with dense coding, and replacing partial freezing with a traditional single layer neural network (i.e., logistic regression), where every weight can change for each input. This gave us four combinations to test. The dense code was calculated in the same way as the sparse code, minus the winner-take-all step. In other words, for each input $x$, we used $\psi(x)$ (Equation~\eqref{eqn:1}, with min-max normalization) as its representation, instead of $\phi(x)$ (Equation~\eqref{eqn:2}). For logistic regression, the associative layer was trained using backpropagation.

Both sparse coding variants (with partial freezing or with logistic regression) performed substantially better than the two dense coding variants on both datasets (Figure~\ref{fig:SC_PF}A--B). For example, on \mnist, at the end of training, the sparse coding models had an average accuracy of 0.64 compared to 0.07 for the two dense coding models. Further, sparse coding with partial freezing (i.e., the \flymodel) performed better than sparse coding with logistic regression: 0.75 vs.\@ 0.54 on \mnist; 0.41 vs.\@ 0.21 on CIFAR-100. 

Hence, on at least the two datasets used here, both sparse coding and partial freezing are needed to optimize continual learning performance.

\subsection*{Empirical and theoretical comparison of the \flymodel with the perceptron}

The fruit fly associative learning algorithm (partial freezing) bears resemblance to a well-known supervised learning algorithm --- the perceptron~\cite{rosenblatt1958perceptron} --- albeit two differences. First, both algorithms increase weights to the correct target MBON (class), but the perceptron also decreases the weights to the incorrect MBON if a mistake is made. Second, the perceptron does not modify weights when a correct prediction is made, whereas partial freezing updates weights even if the correct prediction is made. Next, we continued our exploration of circuit design space by studying how the four combinations of these two rules affect continual learning.

The first model (Perceptron v1) is the classic perceptron learning algorithm, where weights are only modified if an incorrect prediction is made, by increasing weights to the correct class and decreasing weights to the incorrectly predicted class. The second model (Perceptron v2) also only learns when a mistake is made, but it only increases weights to the correct class (i.e., it does not decrease weights to the incorrect class). The third model (Perceptron v3) increases weights to the correct class regardless of whether a mistake is made, and it  decreases weights to the incorrect class when a mistake is made. Finally, the fourth model (Perceptron v4) is equivalent to the \flymodel; it simply increases weights to the correct class regardless of whether a mistake is made. All models start with the same sparse, high-dimensional input representations in the first layer. See Methods for pseudocode for each model.

Overall, we find a striking difference in continual learning with these two tweaks, with the \flymodel performing significantly better than the other three models on both datasets (Figure~\ref{fig:perceptron}A--B). Specifically, learning regardless of whether a mistake is made (v3 and v4) works better than mistake-only learning (v1 and v2), and decreasing the weights to incorrectly predicted class hurts performance (v4 compared to v3; no major difference between v2 and v1). 

Why does decreasing weights to the incorrect class (v1 and v3) result in poor performance? This feature of the perceptron algorithm is believed to help create a larger boundary (margin) between the predicted incorrect class and the correct class. However, in the Supplement (Lemma~2), we show analytically that under continual learning, it is easy to come up with instances where this feature leads to catastrophic forgetting. Intuitively, this occurs when two (similar) inputs share overlapping representations, yet belong to different classes. The synapses of shared neurons are strengthened towards the the class most recently observed, and weakened towards the other class. Thus, when the first input is observed again, it is associated with the second input's class. In other words, decreasing weights to the incorrect class causes shared weights to be ``hijacked'' by recent classes observed (Figure~S2A--C). We tested this empirically on the \mnist dataset and found that, while decreasing weights when mistakes are made enables faster initial discrimination, it also leads to faster forgetting (Figure~S3A). Indeed, this effect is particularly pronounced when the two classes are similar (digits `3' vs.\@ `5'; Figure~S3B) rather than dissimilar (digits `3' vs.\@ `4'; Figure~S3C). In contrast, the \flymodel avoids this issue because the shared neurons are ``split'' between both classes, and thus, cancel each other out (Figure~S2D, Figure~S3).

In support of our empirical findings, we show analytically that partial freezing in the \flymodel (v4) reduces catastrophic forgetting because MBON weight vectors converge over time to the mean of its class inputs scaled by a constant (Supplement, Lemmas~3 and 4, Theorems~5 and 8).

\subsection*{Sparse coding provably creates favorable separation for continual learning}

Why does sparse coding reduce memory interference under continual learning? We first show that the partial freezing algorithm alone will provably learn to correctly distinguish classes if the classes satisfy a {\it separation condition} that says, roughly, that dot products between points within the same class are, on average, greater than between classes. We then show that adding sparse coding enhances the separation of classes~\cite{babadi2014sparseness}, making associative learning easier.
\begin{defn}
Let $\pi_1, \ldots, \pi_k$ be distributions over $\R^d$, corresponding to $k$ classes of data points. We say the classes are $\gamma$-separated, for some $\gamma > 0$, if for any pair of classes $j \neq j'$, and any point $x_o$ from class $j$,
$$ \E_{X \sim \pi_{j}} [x_o \cdot X] \geq \gamma + \E_{X' \sim \pi_{j'}}[x_o \cdot X'] .$$
Here, the notation $\E_{X \sim \pi}$ refers to expected value under a vector $X$ drawn at random from distribution $\pi$. 
\label{def:simple-sep}
\end{defn}

Under $\gamma$-separation, the labeling rule
$$ x \mapsto \argmax_{j} \, w_j \cdot x $$
is a perfect classifier if the $w_j$ (i.e., the KC $\rightarrow$ MBON weight vector for class $j$) are the means of their respective classes, that is, $w_j = \E_{X \sim \pi_j}[X]$. This holds even if the means are only approximately accurate, within $O(\gamma)$ (Supplement, Theorem 8). The partial freezing algorithm can, in turn, be shown to produce such mean-estimates (Supplement, Theorem 5).

The separation condition of Definition~\ref{def:simple-sep} is quite strong and might not hold in the original data space. But we will show that subsequent sparse coding can nonetheless produce this condition, so that the partial freezing algorithm, when run on the sparse encodings, performs well.

To see a simple model of how this can happen, suppose that there are $N$ {\it prototypical inputs}, denoted $p_1, \ldots, p_N \in \X$, where $\X \subset \mathbb{R}^d$, that are somewhat separated from each other:
$$ \frac{p_i \cdot p_j}{\|p_i\| \|p_j\|} \leq \xi, $$
for some $\xi \in [0,1)$. Each $p_i$ has a label $y_i \in \{1,2,\ldots,k\}$. Let $C_j \subset [N]$ be the set of prototypes whose label is $j$. Since the labels are arbitrary, these classes will in general not be linearly separable in the original space~(Figure S4).

Suppose the sparse coding map $\phi: \X \rightarrow \{0,1\}^m$ generates $k$-sparse representations with the following property: for any $x,x' \in \X$,
$$ \phi(x) \cdot \phi(x') \leq k \, f \left( \frac{x \cdot x'}{\|x\| \|x'\|} \right) ,$$
where $f: [-1,1] \rightarrow [0,1]$ is a function that captures how the coding process transforms dot products. In earlier work~\cite{dasgupta2018neural}, we have characterized $f$ for two types of random mappings, a sparse binary matrix (inspired by the fly's architecture) and a dense Gaussian matrix (common in engineering applications). In either case, $f(s)$ is a much shrunken version of $s$; in the dense Gaussian case, for instance, it is roughly $(k/m)^{1-s}$.

We can show that for suitable $\xi$, the sparse representations of the prototypes --- that is, $\phi(p_1), \ldots, \phi(p_N) \in \{0,1\}^m$ --- are then guaranteed to be separable, so that the partial freezing algorithm will converge to a perfect classifier.
\begin{thm}
Let $N_o = \max_j |C_j|$. Under the assumptions above, the sparse representation of the data set, $\{(\phi(p_1), y_1), \ldots, (\phi(p_N), y_N)\}$, is $(1/N_o - f(\xi))$-separated in the sense of Definition~\ref{def:simple-sep}.
\label{thm:sparse-coding-sep-simple}
\end{thm}

\begin{proof}
This is a consequence of Theorem 9 in the Supplement, a more general result that applies to a broader model in which observed data are noisy versions of the prototypes. 
\end{proof}

\section*{Discussion}

We developed a simple and light-weight neural algorithm to alleviate catastrophic forgetting, inspired by how fruit flies learn odor-behavior associations. The \flymodel outperformed three popular class-incremental continual learning algorithms on two benchmark datasets (\mnist and CIFAR-100), despite not using external memory, generative replay, nor backpropagation. We showed that alternative circuits in design space, including the classic perceptron learning rule, suffered more catastrophic forgetting than the \flymodel, potentially shedding new light on the biological function and conservation of this circuit motif. Finally, we grounded these ideas theoretically by proving that MBON weight vectors in the \flymodel converge to the mean representation of its class, and that sparse coding further reduces memory interference by better separating classes. Our work exemplifies how understanding detailed neural anatomy and physiology in a tractable model system can be translated into efficient architectures for use in artificial neural networks. 

The two main features of the \flymodel --- sparse coding and partial synaptic freezing --- are  well-appreciated in both neuroscience and machine learning. For example, sparse, high-dimensional representations have long been recognized as central to neural encoding~\cite{Kanerva1988}, hyper-dimensional computing~\cite{Kanerva2009}, and  classification and recognition tasks~\cite{babadi2014sparseness}. However, the benefits of such representations towards continual learning have not been well-quantified. Similarly, the notion of ``freezing'' certain weights during learning has been used in both classic perceptrons and modern deep networks~\cite{Kirkpatrick2017,Zenke2017}, but these methods are still subject to interference caused by dense representations. Hence, the fruit fly circuit evolved a unique combination of common computational ingredients that work effectively in practice. 

The \flymodel performs associative rather than supervised learning. In associative learning, the same learning rule is applied regardless of whether the model makes a mistake. In traditional supervised learning, changes are only made to weights when the model makes a mistake, and the changes are applied to weights for both the correct and the incorrect class labels. By performing associative learning, the \flymodel garners two benefits. First, the \flymodel learns each class independently compared to supervised methods, which focus on discrimination between multiple classes at a time. We showed that the latter is particularly susceptible to interference, especially when class representations are overlapping. Second, by learning each class independently, the \flymodel is flexible about the total number of classes to be learned; the network is easily expandable to more classes, if necessary. Our results suggest that some traditional benefits of supervised classification may not carry over into the continual learning setting~\cite{Hand2006}, and that association-like models may better preserve memories when classes are learned sequentially. 

There are additional features of the fruit fly circuitry (specifically, the mushroom body) that remain under-explored computationally. First, instead of using one output neuron (MBON) per behavior, the mushroom body contains multiple output neurons per behavior, with each output neuron learning at a different rate~\cite{Hige2015,Aso2016}. This simultaneously provides fast learning with poor retention (large learning rates) and slow learning with longer retention (small learning rates), which is reminiscent of complementary learning systems~\cite{Parisi2019}. Second, the mushroom body contains mechanisms for memory extinction~\cite{Felsenberg2018} and reversal learning~\cite{felsenberg2017re,felsenberg2021changing}, which are used to over-write specific memories that are no longer accurate. Third, there is evidence of memory replay in the mushroom bodytriggered by a neuron called DPM, which is required not during, but rather after memory formation, in order for memories to be consolidated~\cite{yu2005drosophila,Haynes2015,Cognigni2018}.

Beyond catastrophic forgetting, there are additional challenges of continual learning that remain outstanding. These challenges include forward transfer learning (i.e., information learned in the past should help with learning new information more efficiently) and backward transfer learning (i.e., learning new information helps ``deepen'' the understanding of previously learned information). None of the algorithms we tested, including the \flymodel, were specifically designed to addresses these challenges, with the exception of GEM~\cite{lopez2017gradient}, which indeed achieved slightly negative memory losses (i.e., tasks learned after task $i$ improve the accuracy of task $i$; Figure~\ref{fig:continual}D), indicating some success at backward transfer learning. Biologically, circuit mechanisms supporting transfer learning remain unknown.

Finally, a motif similar to that of the fruit fly olfactory system also appears in the mouse olfactory system, where sparse representations in the piriform cortex project to other learning-related areas of the brain~\cite{Komiyama2006,Wang2020}. In addition, the visual system uses many successive layers to extract discriminative features~\cite{Riesenhuber1999,Tacchetti2018}, which are then projected to the hippocampus, where a similar sparse, high-dimensional representation is used for memory storage~\cite{olshausen2004sparse,wixted2014sparse,lodge2019synaptic}. Thus, the principles of learning studied here may help illuminate how continual learning is implemented in other brain regions and species.

\subsection*{Acknowledgment}
\noindent
Y.S. was supported by a Swartz Foundation Fellowship.

\section*{Methods}
\noindent \textbf{Datasets and pre-processing} We tested our model on two datasets.

\textit{\mnist:} This benchmark combines MNIST and Fashion MNIST. For training, MNIST contains 60,000 gray-scale images of 10 classes of hand-written digits (0--9), and Fashion MNIST~\cite{xiao2017} contains 60,000 gray-scale images of 10 classes of fashion items (e.g., purse, pants, etc.). The test set contains 10,000 additional images from each dataset. Together, the two datasets contain 20 classes. The 10 digits in MNIST are labelled 0--9, and the 10 classes in Fashion MNIST are labelled 10--19 in our experiments. To generate a generic input representation for each image, we trained a LeNet5~\cite{lecun1998gradient} network (learning rate = $0.001$, batch size = 64, number of epochs = 25, with batch normalization and Adam) on KMNIST~\cite{clanuwat2018deep}, which contains 60,000 images for 10 classes of hand-written Japanese characters. The penultimate layer of this network contains 84 hidden units (features). We used this trained LeNet5 as an encoder to extract 84 features for each training and test image in MNIST and Fashion MNIST.

\textit{CIFAR-100:} This benchmark contains 50,000 RGB images for 100 classes of real-life objects in the training set, and 10,000 images in the testing set. To generate input representations, we used the penultimate layer (512 hidden nodes) of ResNet18~\cite{HeZRS15} pre-trained on ImageNet (downloaded from \url{https://pytorch.org/docs/stable/torchvision/models.html#id27}). Thus, each CIFAR-100 image was represented as a 512-dimensional vector.\\

\noindent \textbf{Network architectures.} All methods we tested share the same network architecture: a three-layer network with an input layer (analog to PNs in fruit flies), a single hidden layer (analog to KCs) and an output layer (analog to the MBONs). For the \mnist dataset, the network contains 84 nodes in the input layer, 3200 nodes in the hidden layer, and 20 nodes in the output layer. For CIFAR-100, these three numbers are 512, 20000, and 100 respectively. The size of the hidden layer was selected to be approximately 40x larger than the input layer, as per the fly circuit.

For all models except the \flymodel, the three layers make all-to-all connections. For fly model, the PN and KC layer are connected via a sparse random matrix ($\Theta$); each KC sums over 10 randomly selected PNs for \mnist, and 64 randomy PNs for CIFAR-100.\\

\noindent \textbf{Implemenations of other methods.} GEM and EWC implementations are adapted from: \url{https://github.com/facebookresearch/GradientEpisodicMemory}. The BI-R implementation is adapted from: \url{https://github.com/GMvandeVen/brain-inspired-replay}.\\

\noindent \textbf{Parameters.} Parameters for each model and dataset were independently selected using grid search to maximize accuracy.
\begin{itemize}
\item \flymodel: learning rate: 0.01 (\mnist), 0.2 (CIFAR-100); $l = m/k$, where $k=20$, the number of classes for \mnist and $k=100$ for CIFAR-100; and $m=3200$, the number of Kenyon cells for \mnist, and $m=20000$ for CIFAR-100.
\item GEM: learning rate: 0.001, memory strength: 0.5, n memories: 256, batch size: 64, for both datasets.
\item EWC: learning rate: 0.1 (\mnist), 0.001 (CIFAR-100); memory strength: 1000, n memories: 1000, batch size: 64, for both datasets.
\item BI-R: learning rate: 0.001, batch size: 64 for both datasets. The BI-R architecture and other parameters are default to the original implementation~\cite{Ven2020}.
\item Vanilla: learning rate: 0.001, batch size: 64, for both datasets.
\item Offline: learning rate: 0.001, batch size: 64, for both datasets.
\end{itemize}

For Offline, Vanilla, EWC, and GEM, a softmax activation is used for the output layer, and optimization is performed using stochasic gradient descent (SGD). For BI-R, no activation function is applied to the output layer, and optimization is performed using Adam with $\beta_1 = 0.900$, $\beta_2=0.999$.\\

We report the average and standard deviation of both evaluation measures for each method over five random initializations. 


\clearpage
\noindent \textbf{Perceptron variations.} The update rules for the four perceptron variations are listed below:

\noindent \begin{minipage}{0.44\textwidth}
\begin{algorithm}[H]
    \centering
    \caption{\textbf{v1} (Original)}
    \begin{algorithmic}[1]
        \For{$x$ in data}
            \If {predict $\neq$ target}
                \State{weight[target] $\pluseq$ $\beta x$}
                \State{weight[predict] $\minuseq$ $\beta x$}
            \EndIf
        \EndFor
    \end{algorithmic}
\end{algorithm}
\end{minipage}
\hfill
\begin{minipage}{0.44\textwidth}
\begin{algorithm}[H]
    \centering
    \caption{\textbf{v2}}
    \begin{algorithmic}[1]
        \For{$x$ in data}
            \If {predict $\neq$ target}
                \State{weight[target] $\pluseq$ $\beta x$}
                \State
            \EndIf
        \EndFor
    \end{algorithmic}
\end{algorithm}
\end{minipage}

\vspace{0.3in}

\noindent \begin{minipage}{0.44\textwidth}
\begin{algorithm}[H]
    \centering
    \caption{\textbf{v3}}
    \begin{algorithmic}[1]
        \For{$x$ in data}
            \If {predict $\neq$ target}
                \State{weight[target] $\pluseq$ $\beta x$}
                \State{weight[predict] $\minuseq$ $\beta x$}
            \Else
                \State{weight[target] $\pluseq$ $\beta x$}
            \EndIf
        \EndFor
    \end{algorithmic}
\end{algorithm}
\end{minipage}
\hfill
\begin{minipage}{0.44\textwidth}
\begin{algorithm}[H]
    \centering
    \caption{\textbf{v4} (\flymodel)}
    \begin{algorithmic}[1]
        \For{$x$ in data}
            \If {predict $\neq$ target}
                \State{weight[target] $\pluseq$ $\beta x$}
                \State
            \Else
                \State{weight[target] $\pluseq$ $\beta x$}                
            \EndIf
        \EndFor
    \end{algorithmic}
\end{algorithm}
\end{minipage}

\clearpage
\section*{Figures}

\begin{figure*}[h]
\begin{center}
\includegraphics[width=\textwidth]{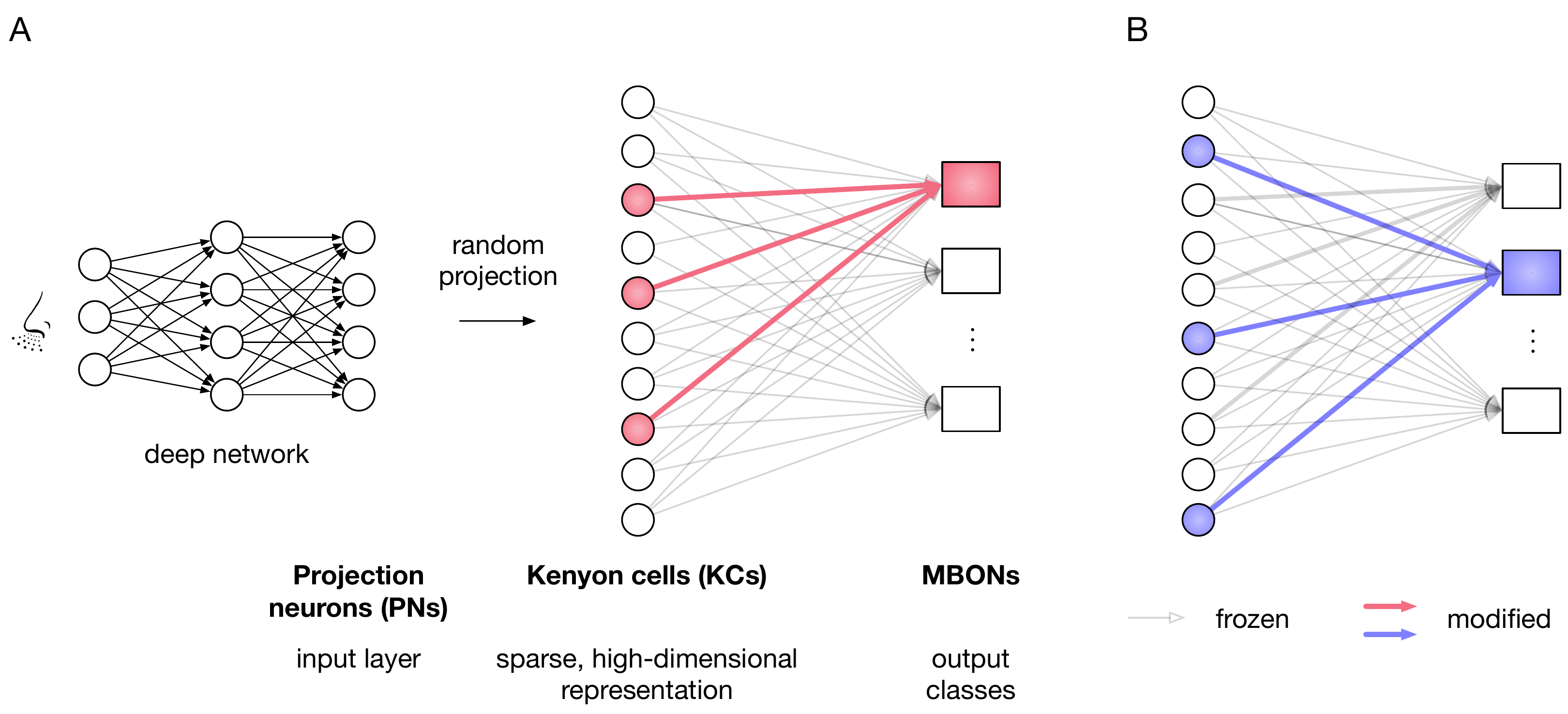}
\end{center}
\vspace{-0.1in}
\caption{\textbf{A two-layer circuit for continual learning in the fruit fly olfactory system.} A) An input (odor) is received by the sensory layer and pre-processed via a series of transformations. In the fruit fly, this pre-processing includes noise reduction, normalization, and gain control. In a deep network, pre-processing is similarly used to generate a suitable representation for learning. After these transformations, the dimensionality of the pre-processed input (PNs) is expanded via a random projection and is sparsified via winner-take-all thresholding. This leaves only a few Kenyon cells active per odor (indicated by red shading). To associate the odor with an output class (MBON), only the synapses connecting the active Kenyon cells to the target MBON are modified. The rest of the synapses are frozen. B) A second example with a second odor, showing different Kenyon cells activated, associated with a different MBON.}
\label{fig:overview}
\end{figure*}

\begin{figure*}[h]
\begin{center}
\includegraphics[width=\textwidth]{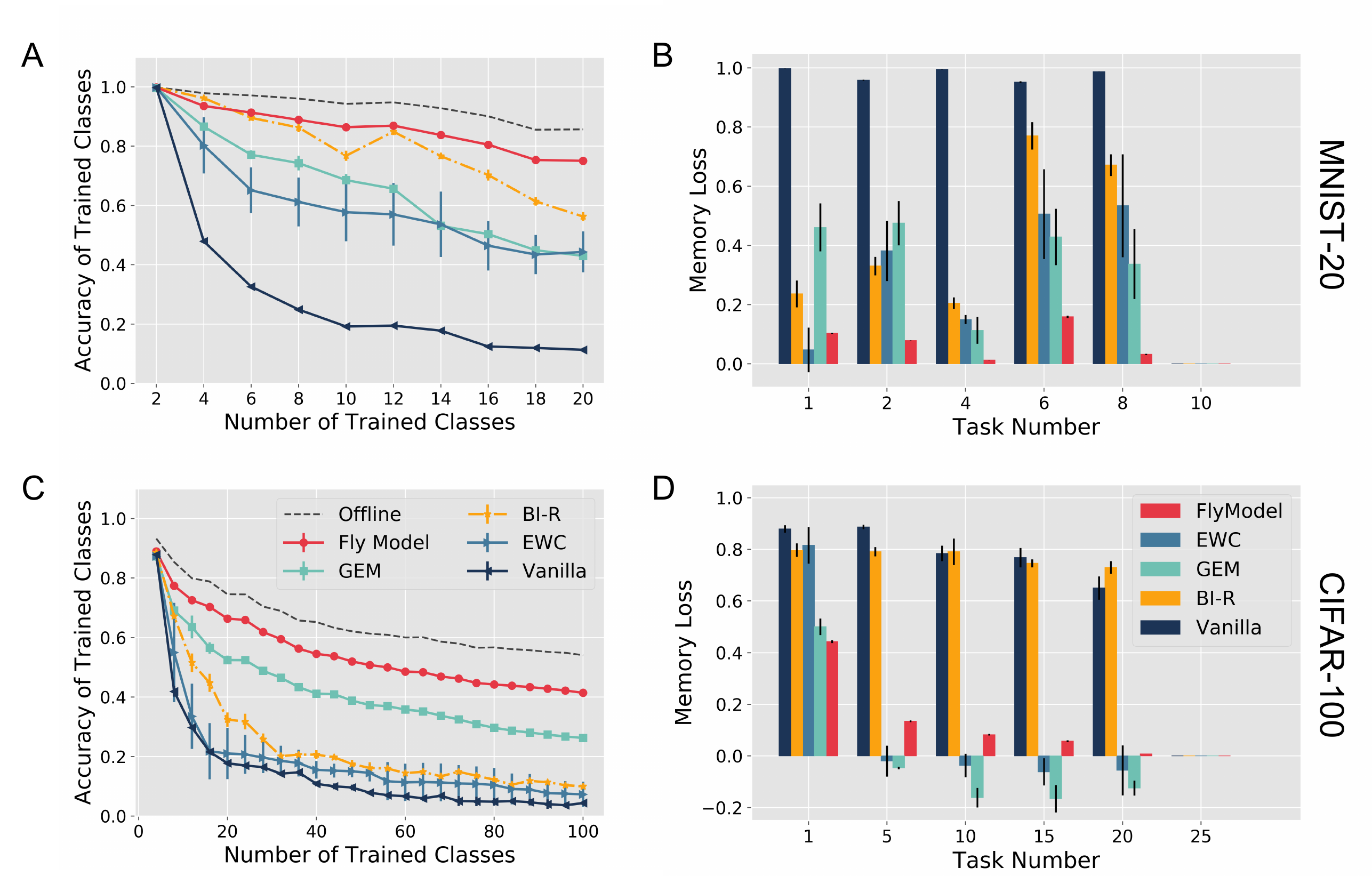}
\end{center}
\vspace{-0.1in}
\caption{\textbf{The \flymodel outperforms existing continual learning methods in class-incremental learning.} A) The $x$-axis is the number of classes trained on, and the $y$-axis is the classification accuracy when testing the model on the classes trained on thus far. The Offline method (dashed black line) shows the optimal classification accuracy when classes are presented together, instead of sequentially. Error bars show standard deviation of the test accuracy over 5 random initializations for GEM, BI-R, EWC, and Vanilla, or over 5 random matrices ($\Theta$) for the \flymodel. B) The $x$-axis is the task number during training, and the $y$-axis is the memory loss of the task, which measures how much the network has forgotten about the task as a result of subsequent training. A--B) \mnist dataset. C--D) CIFAR-100. The memory loss of all tasks is shown in Figure~S1.}
\label{fig:continual}
\end{figure*}

\begin{figure*}[h]
\begin{center}
\includegraphics[width=\textwidth]{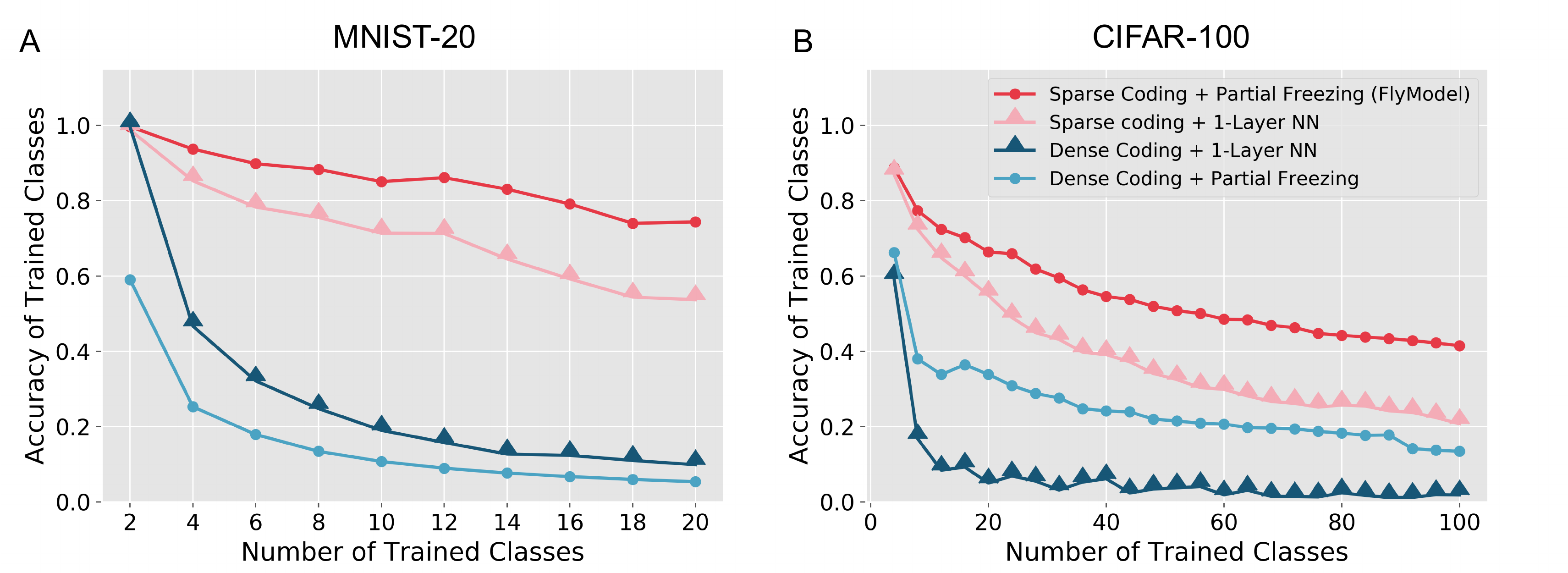}
\end{center}
\vspace{-0.1in}
\caption{\textbf{Sparse coding and partial freezing are both required for continual learning.} Axes are the same as those in Figure 2A. Both sparse coding methods outperform both dense coding methods. When using sparse coding, partial freezing outperforms logistic regression (1-layer neural network). A) \mnist. B) CIFAR-100.}
\label{fig:SC_PF}
\end{figure*}

\begin{figure*}[h]
\begin{center}
\includegraphics[width=\textwidth]{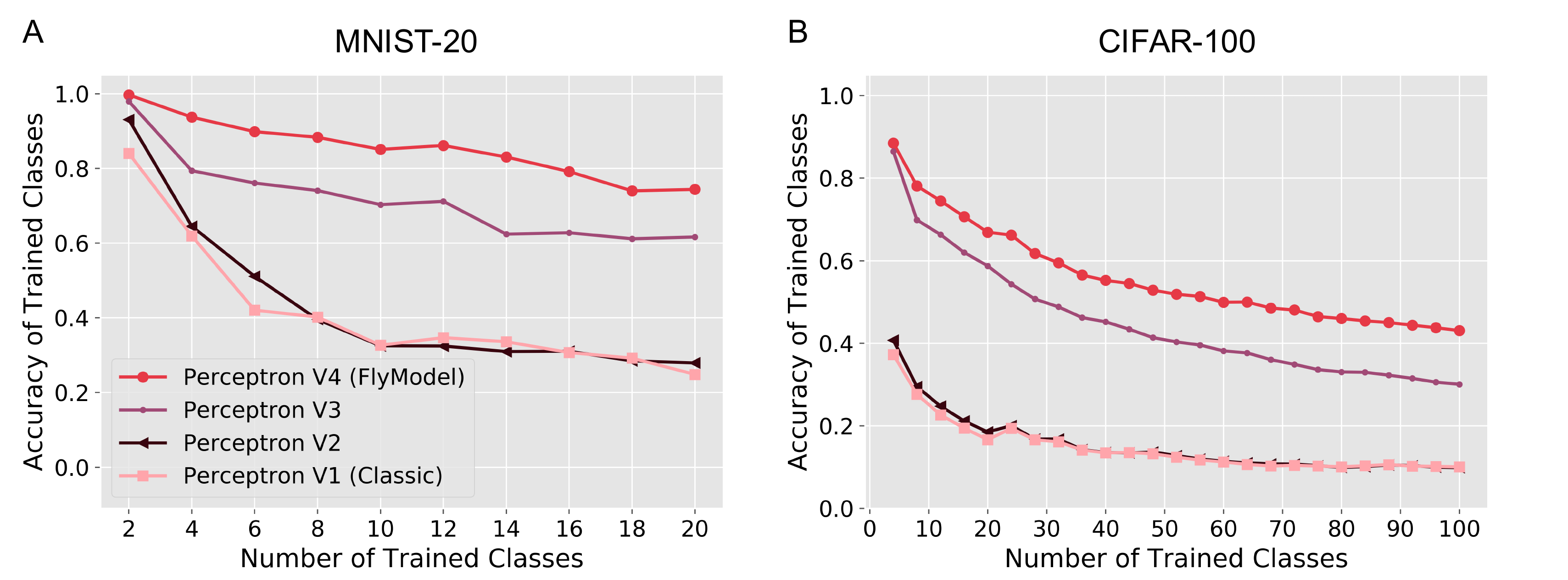}
\end{center}
\vspace{-0.1in}
\caption{\textbf{Continual learning performance for the four perceptron variants.} Axes are the same as those in Figure 2A. Compared to the classic perceptron learning algorithm (Perceptron v1), the \flymodel (Perceptron v4) learns regardless of whether a mistake is made, and it does not decrease weights to the incorrect class when mistakes are made. These two changes significantly improve continual learning performance. A) \mnist. B) CIFAR-100.}
\label{fig:perceptron}
\end{figure*}

\clearpage
\printbibliography

@Article{Peng2017,
   Author="Peng, F.  and Chittka, L. ",
   Title="{{A} {S}imple {C}omputational {M}odel of the {B}ee {M}ushroom {B}ody {C}an {E}xplain {S}eemingly {C}omplex {F}orms of {O}lfactory {L}earning and {M}emory}",
   Journal="Curr Biol",
   Year="2017",
   Volume="27",
   Number="11",
   Pages="1706",
   Month="06"
}

@Article{Mittal2020,
   Author="Mittal, A. M.  and Gupta, D.  and Singh, A.  and Lin, A. C.  and Gupta, N. ",
   Title="{{M}ultiple network properties overcome random connectivity to enable stereotypic sensory responses}",
   Journal="Nat Commun",
   Year="2020",
   Volume="11",
   Number="1",
   Pages="1023",
   Month="02"
}

@article{Hand2006,
author = {David J. Hand},
title = {{Classifier Technology and the Illusion of Progress}},
volume = {21},
journal = {Statistical Science},
number = {1},
publisher = {Institute of Mathematical Statistics},
pages = {1 -- 14},
year = {2006},
doi = {10.1214/088342306000000060},
URL = {https://doi.org/10.1214/088342306000000060}
}

@InProceedings{Hitron2020,
  author =	{Yael Hitron and Nancy Lynch and Cameron Musco and Merav Parter},
  title =	{{Random Sketching, Clustering, and Short-Term Memory in Spiking Neural Networks}},
  booktitle =	{11th Innovations in Theoretical Computer Science Conference (ITCS 2020)},
  pages =	{23:1--23:31},
  ISBN =	{978-3-95977-134-4},
  ISSN =	{1868-8969},
  year =	{2020},
  volume =	{151},
  editor =	{Thomas Vidick},
  publisher =	{Schloss Dagstuhl--Leibniz-Zentrum fuer Informatik},
  address =	{Dagstuhl, Germany}
}

@Article{Rapp2020,
   Author="Rapp, H.  and Nawrot, M. P. ",
   Title="{{A} spiking neural program for sensorimotor control during foraging in flying insects}",
   Journal="Proc Natl Acad Sci U S A",
   Year="2020",
   Volume="117",
   Number="45",
   Pages="28412--28421",
   Month="11"
}

@Article{Haynes2015,
   Author="Haynes, P. R.  and Christmann, B. L.  and Griffith, L. C. ",
   Title="{{A} single pair of neurons links sleep to memory consolidation in {D}rosophila melanogaster}",
   Journal="Elife",
   Year="2015",
   Volume="4",
   Month="Jan"
}

@misc{dasgupta2020expressivity,
      title={Expressivity of expand-and-sparsify representations},
      author={Sanjoy Dasgupta and Christopher Tosh},
      year={2020},
      eprint={2006.03741},
      archivePrefix={arXiv},
      primaryClass={cs.NE}
}

@Article{Komiyama2006,
   Author="Komiyama, T.  and Luo, L. ",
   Title="{{D}evelopment of wiring specificity in the olfactory system}",
   Journal="Curr Opin Neurobiol",
   Year="2006",
   Volume="16",
   Number="1",
   Pages="67--73",
   Month="Feb"
}

@Article{Cervantes2017,
   Author="Cervantes-Sandoval, I.  and Phan, A.  and Chakraborty, M.  and Davis, R. L. ",
   Title="{{R}eciprocal synapses between mushroom body and dopamine neurons form a positive feedback loop required for learning}",
   Journal="Elife",
   Year="2017",
   Volume="6",
   Month="05"
}

@Article{Eichler2017,
   Author="Eichler, K.  and Li, F.  and Litwin-Kumar, A.  and Park, Y.  and Andrade, I.  and Schneider-Mizell, C. M.  and Saumweber, T.  and Huser, A.  and Eschbach, C.  and Gerber, B.  and Fetter, R. D.  and Truman, J. W.  and Priebe, C. E.  and Abbott, L. F.  and Thum, A. S.  and Zlatic, M.  and Cardona, A. ",
   Title="{{T}he complete connectome of a learning and memory centre in an insect brain}",
   Journal="Nature",
   Year="2017",
   Volume="548",
   Number="7666",
   Pages="175--182",
   Month="08"
}

@InProceedings{Ryali2020,
    title = {Bio-Inspired Hashing for Unsupervised Similarity Search},
    author = {Ryali, Chaitanya and Hopfield, John and Grinberg, Leopold and Krotov, Dmitry},
    booktitle = {Proceedings of the 37th International Conference on Machine Learning},
    pages = {8295--8306},
    year = {2020},
    editor = {Hal Daumé III and Aarti Singh}, volume = {119},
    series = {Proceedings of Machine Learning Research},
    month = {13--18 Jul},
    publisher = {PMLR}
}

@Article{Zhang2016,
   Author="Zhang, Y.  and Sharpee, T. O. ",
   Title="{{A} {R}obust {F}eedforward {M}odel of the {O}lfactory {S}ystem}",
   Journal="PLoS Comput Biol",
   Year="2016",
   Volume="12",
   Number="4",
   Pages="e1004850",
   Month="Apr"
}

@Article{Cayco2019,
   Author="Cayco-Gajic, N. A.  and Silver, R. A. ",
   Title="{{R}e-evaluating {C}ircuit {M}echanisms {U}nderlying {P}attern {S}eparation}",
   Journal="Neuron",
   Year="2019",
   Volume="101",
   Number="4",
   Pages="584--602",
   Month="02"
}

@Article{LitwinKumar2017,
   Author="Litwin-Kumar, A.  and Harris, K. D.  and Axel, R.  and Sompolinsky, H.  and Abbott, L. F. ",
   Title="{{O}ptimal {D}egrees of {S}ynaptic {C}onnectivity}",
   Journal="Neuron",
   Year="2017",
   Volume="93",
   Number="5",
   Pages="1153--1164",
   Month="Mar"
}

@Article{Gorur2017,
   Author="Gorur-Shandilya, S.  and Demir, M.  and Long, J.  and Clark, D. A.  and Emonet, T. ",
   Title="{{O}lfactory receptor neurons use gain control and complementary kinetics to encode intermittent odorant stimuli}",
   Journal="Elife",
   Year="2017",
   Volume="6",
   Month="06"
}

@Article{Root2008,
   Author="Root, C. M.  and Masuyama, K.  and Green, D. S.  and Enell, L. E.  and N?ssel, D. R.  and Lee, C. H.  and Wang, J. W. ",
   Title="{{A} presynaptic gain control mechanism fine-tunes olfactory behavior}",
   Journal="Neuron",
   Year="2008",
   Volume="59",
   Number="2",
   Pages="311--321",
   Month="Jul"
}

@book{Minsky1988,
author = {Minsky, Marvin L. and Papert, Seymour A.},
title = {Perceptrons: Expanded Edition},
year = {1988},
isbn = {0262631113},
publisher = {MIT Press},
address = {Cambridge, MA, USA}
}

@Article{Lin2014,
   Author="Lin, A. C.  and Bygrave, A. M.  and de Calignon, A.  and Lee, T.  and Miesenb?ck, G. ",
   Title="{{S}parse, decorrelated odor coding in the mushroom body enhances learned odor discrimination}",
   Journal="Nat Neurosci",
   Year="2014",
   Volume="17",
   Number="4",
   Pages="559--568",
   Month="Apr"
}

@Article{Turner2008,
   Author="Turner, G. C.  and Bazhenov, M.  and Laurent, G. ",
   Title="{{O}lfactory representations by {D}rosophila mushroom body neurons}",
   Journal="J Neurophysiol",
   Year="2008",
   Volume="99",
   Number="2",
   Pages="734--746",
   Month="Feb"
}

@Article{Stevens2015,
   Author="Stevens, C. F. ",
   Title="{{W}hat the fly's nose tells the fly's brain}",
   Journal="Proc Natl Acad Sci U S A",
   Year="2015",
   Volume="112",
   Number="30",
   Pages="9460--9465",
   Month="Jul"
}

@Article{Cognigni2018,
   Author="Cognigni, P.  and Felsenberg, J.  and Waddell, S. ",
   Title="{{D}o the right thing: neural network mechanisms of memory formation, expression and update in {D}rosophila}",
   Journal="Curr Opin Neurobiol",
   Year="2018",
   Volume="49",
   Pages="51--58",
   Month="04"
}

@Article{Modi2020,
   Author="Modi, M. N.  and Shuai, Y.  and Turner, G. C. ",
   Title="{{T}he {D}rosophila {M}ushroom {B}ody: {F}rom {A}rchitecture to {A}lgorithm in a {L}earning {C}ircuit}",
   Journal="Annu Rev Neurosci",
   Year="2020",
   Volume="43",
   Pages="465--484",
   Month="07"
}

@InProceedings{Papadimitriou2018,
  author =	{Christos H. Papadimitriou and Santosh S. Vempala},
  title =	{{Random Projection in the Brain and Computation with Assemblies of Neurons}},
  booktitle =	{10th Innovations in Theoretical Computer Science  Conference (ITCS 2019)},
  pages =	{57:1--57:19},
  series =	{Leibniz International Proceedings in Informatics (LIPIcs)},
  ISBN =	{978-3-95977-095-8},
  ISSN =	{1868-8969},
  year =	{2018},
  volume =	{124},
  editor =	{Avrim Blum},
  publisher =	{Schloss Dagstuhl--Leibniz-Zentrum fuer Informatik},
  address =	{Dagstuhl, Germany},
  doi =		{10.4230/LIPIcs.ITCS.2019.57}
}

@Article{Dasgupta2017,
   Author="Dasgupta, S.  and Stevens, C. F.  and Navlakha, S. ",
   Title="{{A} neural algorithm for a fundamental computing problem}",
   Journal="Science",
   Year="2017",
   Volume="358",
   Number="6364",
   Pages="793--796",
   Month="11"
}

@Article{Wilson2013,
   Author="Wilson, R. I. ",
   Title="{{E}arly olfactory processing in {D}rosophila: mechanisms and principles}",
   Journal="Annu Rev Neurosci",
   Year="2013",
   Volume="36",
   Pages="217--241",
   Month="Jul"
}

@article{Subutai2019,
  author    = {Subutai Ahmad and
               Luiz Scheinkman},
  title     = {How Can We Be So Dense? The Benefits of Using Highly Sparse Representations},
  journal   = {CoRR},
  volume    = {abs/1903.11257},
  year      = {2019},
  url       = {http://arxiv.org/abs/1903.11257},
  archivePrefix = {arXiv},
  eprint    = {1903.11257},
  bibsource = {dblp computer science bibliography, https://dblp.org}
}

@InProceedings{Ruvolo2013,
    title = {{ELLA}: An Efficient Lifelong Learning Algorithm},
    author = {Paul Ruvolo and Eric Eaton},
    booktitle = {Proceedings of the 30th International Conference on Machine Learning},
    pages = {507--515},
    year = {2013},
    editor = {Sanjoy Dasgupta and David McAllester},
    volume = {28},
    number = {1},
    series = {Proceedings of Machine Learning Research},
    address = {Atlanta, Georgia, USA},
    month = {17--19 Jun},
    publisher = {PMLR}
}

@inproceedings{Maurer2013,
author = {Maurer, Andreas and Pontil, Massimiliano and Romera-Paredes, Bernardino},
title = {Sparse Coding for Multitask and Transfer Learning},
year = {2013},
publisher = {JMLR.org},
booktitle = {Proc.\@ 30th Intl.\@ Conf.\@ on Machine Learning - Volume 28},
pages = {II–343–II–351},
location = {Atlanta, GA, USA},
series = {ICML'13}
}

@article{Ororbia2019,
  author    = {Alexander Ororbia and
               Ankur Mali and
               Daniel Kifer and
               C. Lee Giles},
  title     = {Lifelong Neural Predictive Coding: Sparsity Yields Less Forgetting
               when Learning Cumulatively},
  journal   = {CoRR},
  volume    = {abs/1905.10696},
  year      = {2019},
  url       = {http://arxiv.org/abs/1905.10696},
  archivePrefix = {arXiv},
  eprint    = {1905.10696}
}

@Article{Li2020,
   Author="Li, F.  and Lindsey, J. W.  and Marin, E. C.  and Otto, N.  and Dreher, M.  and Dempsey, G.  and Stark, I.  and Bates, A. S.  and Pleijzier, M. W.  and Schlegel, P.  and Nern, A.  and Takemura, S. Y.  and Eckstein, N.  and Yang, T.  and Francis, A.  and Braun, A.  and Parekh, R.  and Costa, M.  and Scheffer, L. K.  and Aso, Y.  and Jefferis, G. S.  and Abbott, L. F.  and Litwin-Kumar, A.  and Waddell, S.  and Rubin, G. M. ",
   Title="{{T}he connectome of the adult {D}rosophila mushroom body provides insights into function}",
   Journal="Elife",
   Year="2020",
   Volume="9",
   Month="Dec"
}

@Article{Zheng2018,
   Author="Zheng, Z.  and Lauritzen, J. S.  and Perlman, E.  and Robinson, C. G.  and Nichols, M.  and Milkie, D.  and Torrens, O.  and Price, J.  and Fisher, C. B.  and Sharifi, N.  and Calle-Schuler, S. A.  and Kmecova, L.  and Ali, I. J.  and Karsh, B.  and Trautman, E. T.  and Bogovic, J. A.  and Hanslovsky, P.  and Jefferis, G. S. X. E.  and Kazhdan, M.  and Khairy, K.  and Saalfeld, S.  and Fetter, R. D.  and Bock, D. D. ",
   Title="{{A} {C}omplete {E}lectron {M}icroscopy {V}olume of the {B}rain of {A}dult {D}rosophila melanogaster}",
   Journal="Cell",
   Year="2018",
   Volume="174",
   Number="3",
   Pages="730--743",
   Month="07"
}

@Article{Takemura2017,
   Author="Takemura, S. Y.  and Aso, Y.  and Hige, T.  and Wong, A.  and Lu, Z.  and Xu, C. S.  and Rivlin, P. K.  and Hess, H.  and Zhao, T.  and Parag, T.  and Berg, S.  and Huang, G.  and Katz, W.  and Olbris, D. J.  and Plaza, S.  and Umayam, L.  and Aniceto, R.  and Chang, L. A.  and Lauchie, S.  and Ogundeyi, O.  and Ordish, C.  and Shinomiya, A.  and Sigmund, C.  and Takemura, S.  and Tran, J.  and Turner, G. C.  and Rubin, G. M.  and Scheffer, L. K. ",
   Title="{{A} connectome of a learning and memory center in the adult {D}rosophila brain}",
   Journal="Elife",
   Year="2017",
   Volume="6",
   Month="07"
}

@Article{Ji2007,
   Author="Ji, D.  and Wilson, M. A. ",
   Title="{{C}oordinated memory replay in the visual cortex and hippocampus during sleep}",
   Journal="Nat Neurosci",
   Year="2007",
   Volume="10",
   Number="1",
   Pages="100--107",
   Month="Jan"
}

@Article{Qin1997,
   Author="Qin, Y. L.  and McNaughton, B. L.  and Skaggs, W. E.  and Barnes, C. A. ",
   Title="{{M}emory reprocessing in corticocortical and hippocampocortical neuronal ensembles}",
   Journal="Philos Trans R Soc Lond B Biol Sci",
   Year="1997",
   Volume="352",
   Number="1360",
   Pages="1525--1533",
   Month="Oct"
}

@Article{Rasch2007,
   Author="Rasch, B.  and Born, J. ",
   Title="{{M}aintaining memories by reactivation}",
   Journal="Curr Opin Neurobiol",
   Year="2007",
   Volume="17",
   Number="6",
   Pages="698--703",
   Month="Dec"
}

@Article{Wilson1994,
   Author="Wilson, M. A.  and McNaughton, B. L. ",
   Title="{{R}eactivation of hippocampal ensemble memories during sleep}",
   Journal="Science",
   Year="1994",
   Volume="265",
   Number="5172",
   Pages="676--679",
   Month="Jul"
}

@Article{McClelland1995,
   Author="McClelland, J. L.  and McNaughton, B. L.  and O'Reilly, R. C. ",
   Title="{{W}hy there are complementary learning systems in the hippocampus and neocortex: insights from the successes and failures of connectionist models of learning and memory}",
   Journal="Psychol Rev",
   Year="1995",
   Volume="102",
   Number="3",
   Pages="419--457",
   Month="Jul"
}

@Article{Roxin2013,
   Author="Roxin, A.  and Fusi, S. ",
   Title="{{E}fficient partitioning of memory systems and its importance for memory consolidation}",
   Journal="PLoS Comput Biol",
   Year="2013",
   Volume="9",
   Number="7",
   Pages="e1003146"
}

@inproceedings{Tadros2020,
  author    = {Timothy Tadros and
               Giri P. Krishnan and
               Ramyaa Ramyaa and
               Maxim Bazhenov},
  title     = {Biologically Inspired Sleep Algorithm for Reducing Catastrophic Forgetting in Neural Networks},
  booktitle = {The Thirty-Fourth {AAAI} Conference on Artificial Intelligence, {AAAI}, New York, NY, USA},
  pages     = {13933--13934},
  publisher = {{AAAI} Press},
  year      = {2020}
}

@Article{Zenke2017,
   Author="Zenke, F.  and Poole, B.  and Ganguli, S. ",
   Title="{{C}ontinual {L}earning {T}hrough {S}ynaptic {I}ntelligence}",
   Journal="Proc Mach Learn Res",
   Year="2017",
   Volume="70",
   Pages="3987--3995"
}

@Article{Fusi2005,
   Author="Fusi, S.  and Drew, P. J.  and Abbott, L. F. ",
   Title="{{C}ascade models of synaptically stored memories}",
   Journal="Neuron",
   Year="2005",
   Volume="45",
   Number="4",
   Pages="599--611",
   Month="Feb"
}

@Article{Benna2016,
   Author="Benna, M. K.  and Fusi, S. ",
   Title="{{C}omputational principles of synaptic memory consolidation}",
   Journal="Nat Neurosci",
   Year="2016",
   Volume="19",
   Number="12",
   Pages="1697--1706",
   Month="12"
}

@Article{Ratcliff1990,
   Author="Ratcliff, R. ",
   Title="{{C}onnectionist models of recognition memory: constraints imposed by learning and forgetting functions}",
   Journal="Psychol Rev",
   Year="1990",
   Volume="97",
   Number="2",
   Pages="285--308",
   Month="Apr"
}

@Article{French1999,
   Author="French, R. M. ",
   Title="{{C}atastrophic forgetting in connectionist networks}",
   Journal="Trends Cogn Sci",
   Year="1999",
   Volume="3",
   Number="4",
   Pages="128--135",
   Month="Apr"
}

@Article{Ven2020,
   Author="van de Ven, G. M.  and Siegelmann, H. T.  and Tolias, A. S. ",
   Title="{{B}rain-inspired replay for continual learning with artificial neural networks}",
   Journal="Nat Commun",
   Year="2020",
   Volume="11",
   Number="1",
   Pages="4069",
   Month="08"
}

@Article{Lecun2015,
   Author="LeCun, Y.  and Bengio, Y.  and Hinton, G. ",
   Title="{{D}eep learning}",
   Journal="Nature",
   Year="2015",
   Volume="521",
   Number="7553",
   Pages="436--444",
   Month="May"
}

@Article{Felsenberg2018,
   Author="Felsenberg, J.  and Jacob, P. F.  and Walker, T.  and Barnstedt, O.  and Edmondson-Stait, A. J.  and Pleijzier, M. W.  and Otto, N.  and Schlegel, P.  and Sharifi, N.  and Perisse, E.  and Smith, C. S.  and Lauritzen, J. S.  and Costa, M.  and Jefferis, G. S. X. E.  and Bock, D. D.  and Waddell, S. ",
   Title="{{I}ntegration of {P}arallel {O}pposing {M}emories {U}nderlies {M}emory {E}xtinction}",
   Journal="Cell",
   Year="2018",
   Volume="175",
   Number="3",
   Pages="709--722",
   Month="10"
}

@inproceedings{Hinton1987,
    author = {Geoffrey E. Hinton and David C. Plaut},
    title = {Using Fast Weights to Deblur Old Memories},
    booktitle = {Proc.\@ of the 9th Annual Conf.\@ of the Cognitive Science Society},
    year = {1987},
    pages = {177--186},
    publisher = {Erlbaum}
}

@Article{Hige2015,
   Author="Hige, T.  and Aso, Y.  and Modi, M. N.  and Rubin, G. M.  and Turner, G. C. ",
   Title="{{H}eterosynaptic {P}lasticity {U}nderlies {A}versive {O}lfactory {L}earning in {D}rosophila}",
   Journal="Neuron",
   Year="2015",
   Volume="88",
   Number="5",
   Pages="985--998",
   Month="Dec"
}

@Article{Aso2016,
   Author="Aso, Y.  and Rubin, G. M. ",
   Title="{{D}opaminergic neurons write and update memories with cell-type-specific rules}",
   Journal="Elife",
   Year="2016",
   Volume="5",
   Month="07"
}

@Article{Wang2020,
   Author="Wang, P. Y.  and Boboila, C.  and Chin, M.  and Higashi-Howard, A.  and Shamash, P.  and Wu, Z.  and Stein, N. P.  and Abbott, L. F.  and Axel, R. ",
   Title="{{T}ransient and {P}ersistent {R}epresentations of {O}dor {V}alue in {P}refrontal {C}ortex}",
   Journal="Neuron",
   Year="2020",
   Volume="108",
   Number="1",
   Pages="209--224",
   Month="Oct"
}

@Article{Tacchetti2018,
   Author="Tacchetti, A.  and Isik, L.  and Poggio, T. A. ",
   Title="{{I}nvariant {R}ecognition {S}hapes {N}eural {R}epresentations of {V}isual {I}nput}",
   Journal="Annu Rev Vis Sci",
   Year="2018",
   Volume="4",
   Pages="403--422",
   Month="09"
}

@Article{Riesenhuber1999,
   Author="Riesenhuber, M.  and Poggio, T. ",
   Title="{{H}ierarchical models of object recognition in cortex}",
   Journal="Nat Neurosci",
   Year="1999",
   Volume="2",
   Number="11",
   Pages="1019--1025",
   Month="Nov"
}

@Article{Kirkpatrick2017,
   Author="Kirkpatrick, J.  and Pascanu, R.  and Rabinowitz, N.  and Veness, J.  and Desjardins, G.  and Rusu, A. A.  and Milan, K.  and Quan, J.  and Ramalho, T.  and Grabska-Barwinska, A.  and Hassabis, D.  and Clopath, C.  and Kumaran, D.  and Hadsell, R. ",
   Title="{{O}vercoming catastrophic forgetting in neural networks}",
   Journal="Proc Natl Acad Sci U S A",
   Year="2017",
   Volume="114",
   Number="13",
   Pages="3521--3526",
   Month="03"
}

@book{Kanerva1988,
author = {Kanerva, Pentti},
title = {Sparse Distributed Memory},
year = {1988},
isbn = {0262111322},
publisher = {MIT Press},
address = {Cambridge, MA, USA},
}

@article{Kanerva2009,
  year = {2009},
  month = jan,
  publisher = {Springer},
  volume = {1},
  number = {2},
  pages = {139--159},
  author = {Pentti Kanerva},
  title = {Hyperdimensional Computing: An Introduction to Computing in Distributed Representation with High-Dimensional Random Vectors},
  journal = {Cogn.\@ Comput.}
}

@Article{Parisi2019,
   Author="Parisi, G. I.  and Kemker, R.  and Part, J. L.  and Kanan, C.  and Wermter, S. ",
   Title="{{C}ontinual lifelong learning with neural networks: {A} review}",
   Journal="Neural Netw",
   Year="2019",
   Volume="113",
   Pages="54--71",
   Month="May"
}

@inproceedings{lopez2017gradient,
  title={Gradient episodic memory for continual learning},
  author={Lopez-Paz, David and Ranzato, Marc'Aurelio},
  booktitle={Advances in neural information processing systems},
  pages={6467--6476},
  year={2017}
}

@article{caron2013random,
  title={Random convergence of olfactory inputs in the Drosophila mushroom body},
  author={Caron, Sophie JC and Ruta, Vanessa and Abbott, LF and Axel, Richard},
  journal={Nature},
  volume={497},
  number={7447},
  pages={113--117},
  year={2013},
  publisher={Nature Publishing Group}
}

@Article{dasgupta2018neural,
   Author="Dasgupta, S.  and Sheehan, T. C.  and Stevens, C. F.  and Navlakha, S. ",
   Title="{{A} neural data structure for novelty detection}",
   Journal="Proc Natl Acad Sci U S A",
   Year="2018",
   Volume="115",
   Number="51",
   Pages="13093--13098",
   Month="12"
}

@article{olsen2010divisive,
  title={Divisive normalization in olfactory population codes},
  author={Olsen, Shawn R and Bhandawat, Vikas and Wilson, Rachel I},
  journal={Neuron},
  volume={66},
  number={2},
  pages={287--299},
  year={2010},
  publisher={Elsevier}
}

@article{aso2014neuronal,
  title={The neuronal architecture of the mushroom body provides a logic for associative learning},
  author={Aso, Yoshinori and Hattori, Daisuke and Yu, Yang and Johnston, Rebecca M and Iyer, Nirmala A and Ngo, Teri-TB and Dionne, Heather and Abbott, LF and Axel, Richard and Tanimoto, Hiromu and others},
  journal={elife},
  volume={3},
  pages={e04577},
  year={2014},
  publisher={eLife Sciences Publications Limited}
}

@online{xiao2017,
  author       = {Han Xiao and Kashif Rasul and Roland Vollgraf},
  title        = {Fashion-MNIST: a Novel Image Dataset for Benchmarking Machine Learning Algorithms},
  date         = {2017-08-28},
  year         = {2017},
  eprintclass  = {cs.LG},
  eprinttype   = {arXiv},
  eprint       = {cs.LG/1708.07747},
}

@article{lecun1998gradient,
  title={Gradient-based learning applied to document recognition},
  author={LeCun, Yann and Bottou, L{\'e}on and Bengio, Yoshua and Haffner, Patrick},
  journal={Proceedings of the IEEE},
  volume={86},
  number={11},
  pages={2278--2324},
  year={1998},
  publisher={Ieee}
}

@online{clanuwat2018deep,
  author       = {Tarin Clanuwat and Mikel Bober-Irizar and Asanobu Kitamoto and Alex Lamb and Kazuaki Yamamoto and David Ha},
  title        = {Deep Learning for Classical Japanese Literature},
  date         = {2018-12-03},
  year         = {2018},
  eprintclass  = {cs.CV},
  eprinttype   = {arXiv},
  eprint       = {cs.CV/1812.01718},
}

@article{HeZRS15,
  author    = {Kaiming He and
               Xiangyu Zhang and
               Shaoqing Ren and
               Jian Sun},
  title     = {Deep Residual Learning for Image Recognition},
  journal   = {CoRR},
  volume    = {abs/1512.03385},
  year      = {2015},
  url       = {http://arxiv.org/abs/1512.03385},
  archivePrefix = {arXiv},
  eprint    = {1512.03385},
  bibsource = {dblp computer science bibliography, https://dblp.org}
}

@article{olshausen2004sparse,
  title={Sparse coding of sensory inputs},
  author={Olshausen, Bruno A and Field, David J},
  journal={Current opinion in neurobiology},
  volume={14},
  number={4},
  pages={481--487},
  year={2004},
  publisher={Elsevier}
}

@article{wixted2014sparse,
  title={Sparse and distributed coding of episodic memory in neurons of the human hippocampus},
  author={Wixted, John T and Squire, Larry R and Jang, Yoonhee and Papesh, Megan H and Goldinger, Stephen D and Kuhn, Joel R and Smith, Kris A and Treiman, David M and Steinmetz, Peter N},
  journal={Proceedings of the National Academy of Sciences},
  volume={111},
  number={26},
  pages={9621--9626},
  year={2014},
  publisher={National Acad Sciences}
}

@article{lodge2019synaptic,
  title={Synaptic properties of newly generated granule cells support sparse coding in the adult hippocampus},
  author={Lodge, Meredith and Bischofberger, Josef},
  journal={Behavioural brain research},
  volume={372},
  pages={112036},
  year={2019},
  publisher={Elsevier}
}

@article{farquhar2019towards,
  title={Towards robust evaluations of continual learning},
  author={Farquhar, Sebastian and Gal, Yarin},
  journal={arXiv preprint arXiv:1805.09733v3},
  year={2019}
}

@article{rosenblatt1958perceptron,
  title={The perceptron: a probabilistic model for information storage and organization in the brain.},
  author={Rosenblatt, Frank},
  journal={Psychological review},
  volume={65},
  number={6},
  pages={386},
  year={1958},
  publisher={American Psychological Association}
}

@article{babadi2014sparseness,
  title={Sparseness and expansion in sensory representations},
  author={Babadi, Baktash and Sompolinsky, Haim},
  journal={Neuron},
  volume={83},
  number={5},
  pages={1213--1226},
  year={2014},
  publisher={Elsevier}
}

@article{shin2017continual,
  title={Continual learning with deep generative replay},
  author={Shin, Hanul and Lee, Jung Kwon and Kim, Jaehong and Kim, Jiwon},
  journal={arXiv preprint arXiv:1705.08690v3},
  year={2017}
}

@article{felsenberg2017re,
  title={Re-evaluation of learned information in Drosophila},
  author={Felsenberg, Johannes and Barnstedt, Oliver and Cognigni, Paola and Lin, Suewei and Waddell, Scott},
  journal={Nature},
  volume={544},
  number={7649},
  pages={240--244},
  year={2017},
  publisher={Nature Publishing Group}
}

@article{yu2005drosophila,
  title={Drosophila DPM neurons form a delayed and branch-specific memory trace after olfactory classical conditioning},
  author={Yu, Dinghui and Keene, Alex C and Srivatsan, Anjana and Waddell, Scott and Davis, Ronald L},
  journal={Cell},
  volume={123},
  number={5},
  pages={945--957},
  year={2005},
  publisher={Elsevier}
}

@article{felsenberg2021changing,
  title={Changing memories on the fly: the neural circuits of memory re-evaluation in Drosophila melanogaster},
  author={Felsenberg, Johannes},
  journal={Current Opinion in Neurobiology},
  volume={67},
  pages={190--198},
  year={2021},
  publisher={Elsevier}
}

\end{document}